\documentclass[letterpaper,11pt]{article}
\usepackage{fullpage}
\usepackage{microtype}
\usepackage{graphicx}
\usepackage{subfigure}
\usepackage[numbers]{natbib}
\usepackage{ifthen}

\usepackage{doi}
\usepackage{amsthm}
\usepackage{xcolor}
\usepackage{amsmath,amssymb, mathtools}
\usepackage{centernot}

\usepackage{lipsum}

\newcommand\blfootnote[1]{%
  \begingroup
  \renewcommand\thefootnote{}\footnote{#1}%
  \addtocounter{footnote}{-1}%
  \endgroup
}
\newcommand{\mal}{\wt}
\newcommand{\dist}{\mathbf{d}}

\newcommand{\C}{\cC}
\newcommand{\D}{\cD}
\renewcommand{\H}{\cH}
\newcommand{\error}{\mathsf{Err}}

\newcommand{\alphaVar}{{\bm \upalpha}}
\newcommand{\Risk}{\mathsf{Risk}}

\newcommand{\Tam}{\mathsf{T}}

\newcommand{\High}{\mathrm{High}}
\newcommand{\Low}{\mathrm{Low}}
\newcommand{\BigOut}{\mathrm{Big}}
\newcommand{\SmallOut}{\mathrm{Small}}

\newcommand{\xVar}{\mathbf{x}}
\newcommand{\xDist}{\xVar}

\newcommand{\yVar}{\mathbf{y}}
\newcommand{\yDist}{\yVar}

\newcommand{\zVar}{\mathbf{z}}
\newcommand{\xVecVar}{\ol{\xVar}}
\newcommand{\yVecVar}{\ol{\yVar}}
\newcommand{\zVecVar}{\ol{\zVar}}

\newcommand{\xVec}{\ol{x}}
\newcommand{\yVec}{\ol{y}}

\newcommand{\SVar}{\mathbf{S}}
\renewcommand{\th}{^\mathrm{th}}
\newcommand{\twist}[3]{\langle #1 \,\, \| \, {#2}, {#3}\rangle}

\newcommand{\VP}{\mathrm{ValPref}}

\newcommand{\pfix}[2]{ #1_{\leq #2}}

\newcommand{\fhat}[2]{\ifthenelse{\equal{#2}{}}{\hat{f}(#1)}{\ifthenelse{\equal{#2}{0}}{\hat{f}(\emptyset)}{\hat{f}(#1_{\leq #2})}}}

\newcommand{\gain}[2]{\ifthenelse{\equal{#2}{}}{g[#1]}{g[#1_{\leq #2}]}}

\newcommand{\RTam}{\mathsf{RejSam}}
\newcommand{\RTamk}{\mathsf{RejSam}_{k}}

\newcommand{\pr}[2][]{\Pr_{\ifthenelse{\isempty{#1}}{}{{#1}}}\left[{#2}\right]}

\newcommand{\Tamp}{\Tam}

\newcommand{\Samp}{\mathsf{OnSam}}

\newcommand{\sm}{\setminus}

\newcommand{\Rplus}{\R_+}
\newcommand{\Loss}{\mathsf{Loss}}

\newcommand{\E}{\mathrm{E}}

\renewcommand{\D}{\cD}

\usepackage{graphicx}

\newcommand{\remove}[1]{}

\newcommand{\ol}{\overline}
\newcommand{\wt}[1]{\widetilde{#1}}

\newcommand{\se}{\subseteq}

\newcommand{\set}[1]{\{ #1 \}}

\newcommand{\To}{\mapsto}

\newcommand{\R}{{\mathbb R}}
\newcommand{\N}{{\mathbb N}}

\newcommand{\Adv}{\mathrm{Adv}}

\newcommand{\cA}{{\mathcal A}}
\newcommand{\cB}{{\mathcal B}}
\newcommand{\cC}{{\mathcal C}}
\newcommand{\cD}{{\mathcal D}}

\newcommand{\cH}{{\mathcal H}}

\newcommand{\cP}{{\mathcal P}}

\newcommand{\cT}{{\mathcal T}}

\newcommand{\cZ}{{\mathcal Z}}

\newcommand{\bfx}{\mathbf{x}}

\usepackage{upgreek}
\usepackage{bm}

\newcommand{\eps}{\varepsilon}
\newcommand{\vphi}{\varphi}

\newcommand{\poly}{\operatorname{poly}}

\newcommand{\Exp}{\operatorname*{\mathbb{E}}}
\newcommand{\Ex}{\Exp}

\newcommand{\Var}{\operatorname*{\mathbb{V}}}

\newcommand{\Supp}{\operatorname{Supp}}

\newtheorem{theorem}{Theorem}[section]
\newtheorem{claim}[theorem]{Claim}

\newtheorem{lemma}[theorem]{Lemma}
\newtheorem{corollary}[theorem]{Corollary}

\newtheorem{definition}[theorem]{Definition}
\newtheorem{construction}[theorem]{Construction}

\newtheorem{remark}[theorem]{Remark}

\newcommand{\sdotfill}{\textcolor[rgb]{0.8,0.8,0.8}{\dotfill}} 

\makeatletter
\def\th@protocol{%
    \normalfont 
    \setbeamercolor{block title example}{bg=orange,fg=white}
    \setbeamercolor{block body example}{bg=orange!20,fg=black}
    \def\inserttheoremblockenv{exampleblock}
  }
\makeatother

\newtheorem{proto}[theorem]{Protocol}
\newtheorem{protoc}[theorem]{Protocol}

\newcommand{\namedref}[2]{#1~\ref{#2}}

\newcommand{\torestate}[3]{%
\expandafter \def \csname BBRESTATE #2 \endcsname{#3}
\newtheorem{BBRESTATETHMNUM#2}[theorem]{#1}
\begin{BBRESTATETHMNUM#2}\label{#2}\csname BBRESTATE #2 \endcsname   \end{BBRESTATETHMNUM#2}
\newtheorem*{BBRESTATETHMNONNUM#2}{\namedref{#1}{#2}}
}

\newcommand{\restate}[1]{\begin{BBRESTATETHMNONNUM#1}[Restated] \csname BBRESTATE #1 \endcsname
\end{BBRESTATETHMNONNUM#1}}

\title{Universal Multi-Party Poisoning Attacks
\blfootnote{This is a full version of a paper that was previously published in the Proceedings of the 36th International Conference on Machine
Learning, Long Beach, California, PMLR 97, 2019.}}
\author{Saeed Mahloujifar\thanks{\texttt{sfar@princeton.edu}} \and Mohammad Mahmoody\thanks{\texttt{mohammad@virginia.edu}} \and {Ameer Mohammed\thanks{\texttt{ameer.mohammed@ku.edu.kw}}}}

\begin{document}

\maketitle


\vskip 0.3in




\begin{abstract}
In this work, we demonstrate universal multi-party poisoning attacks that  adapt and apply to any multi-party learning process with arbitrary interaction pattern between the parties. More generally, we introduce and study $(k,p)$-poisoning attacks in which an adversary controls $k\in[m]$ of the parties, and  for each corrupted party $P_i$, the adversary submits some poisoned data $\cT'_i$ on behalf of $P_i$ that is still  ``$(1-p)$-close'' to the correct data $\cT_i$ (e.g., $1-p$ fraction of $\cT'_i$ is still honestly generated). We prove that for any ``bad'' property $B$ of the final trained hypothesis $h$ (e.g., $h$ failing on a particular test example or having ``large'' risk) that has an arbitrarily small constant probability of happening without the attack, there  always is a $(k,p)$-poisoning attack that increases the probability of $B$ from $\mu$ to  by $\mu^{1-p \cdot k/m} = \mu + \Omega(p \cdot k/m)$. Our  attack only uses clean  labels, and it is online.

More generally, we prove that for any bounded function $f(x_1,\dots,x_n) \in [0,1]$ defined over an $n$-step random process $\bfx = (x_1,\dots,x_n)$,  an adversary who can override each of the $n$ blocks with \emph{even dependent} probability $p$  can  increase the expected output by at least $\Omega(p \cdot \mathrm{Var}[f(\xVecVar)])$.
\end{abstract}

\tableofcontents

\section{Introduction}
Learning from a  set $\cT=\set{d_1=(a_1,b_1),\dots,d_n=(a_n,b_n)}$ of training examples  in a way that the predictions generalize to instances beyond $\cT$ is a fundamental problem in learning theory. The goal here is to produce a hypothesis   $h$ in such a way that $h(a)$, with high probability, predicts the ``correct'' label $b$, where the pair $(a,b)=d$ is sampled from the target (test) distribution $\dist$. In the most natural setting, the examples in the training data set $\cT$ are also generated from the same distribution $\dist$, however this is not always the case (e.g., due to noise in the data). 

\paragraph{Poisoning attacks.} 
Many previous works studying noise in the data allow it to be \emph{adversarial} and \emph{maliciously} chosen against the learner \citep{Valiant::DisjunctionsConjunctions,KearnsLi::Malicious,NastyNoise}.  A tightly related and more recent approach to the problem of  learning under adversarial noise is the framework of  so-called \emph{poisoning} (aka \emph{causative}) attacks~\citep{barreno2006can,biggio2012poisoning,papernot2016towards}, in which the adversary's goal is not necessarily to completely prevent the learning, but perhaps it simply wants to  increase the risk of the hypothesis produced by the learning process or make it more likely to fail on a particular test instance (i.e., getting a \emph{targeted} poisoning attack \citep{barreno2006can,shen2016uror}).




\paragraph{Multi-party poisoning.} In the distributed setting \citep{mcmahan2017federated,mcmahan2016communication,bonawitz2017practical,konevcny2016federated},  the training data $\cT$ might be coming from various sources; e.g., it can be generated by $m$ data providers $P_1,\dots,P_m$ in an online way, while at the end a fixed algorithm, called the aggregator $G$, generates the hypothesis $h$ based on $\cT$. The goal of $P_1,\dots,P_m$ is to eventually help $G$ construct a hypothesis $h$ that does well (e.g. in the case of classification) in predicting the label $b$ of a given instance $a$, where $(a,b) \gets \dist$ is sampled from the final test distribution. The data provided by each party $P_i$ might even be of ``different type'', so we cannot simply assume that the data provided by $P_i$ is necessarily sampled from the same distribution $\dist$. To model this more general setting, we let $\dist_i$ model the distribution from which the training data $\cT_i$ (of $P_i$) is sampled. Poisoning attacks can naturally be defined in the distributed setting  as well  \citep{fung2018mitigating,bagdasaryan2018backdoor,blanchard2017machine,hayes2018contamination} to model adversaries who partially control the training data $\cT$.
These works, however, focus on attacking and defending specific learning tasks.
This leads us to the central question of this work.
\begin{quote}
    \emph{What is the \emph{inherent} provable power of poisoning attacks in the multi-party setting?}
\end{quote}
Answering the above question is critical for understanding the \emph{limits} of provable security against multi-party poisoning.

\subsection{Our Contribution}
We first formalize a new general model multi-party poisoning. We then prove the existence of universal data poisoning attacks in the multi-party setting that apply to any task.


\paragraph{New attack model: $(k,p)$-poisoning attacks.}  our first contribution of this work is to formalize a general notion that covers multi-party poisoning attackers that corrupt $k$ out of $m$ data provider parties and furthermore, for each message sent by a corrupted party, the adversary still generates data that is ``close'' to the honestly generated data. More formally, a $(k,p)$-poisoning attacker $\Adv$ can first choose to corrupt $k$ of the parties. Then, if a corrupted $\mal{P}_i$ is supposed to send the next message, then the adversary will sample $d \gets \mal{\dist}$ for a maliciously chosen distribution $\mal{\dist}$ that is guaranteed to be $p$ to the original distribution $\dist_i$ in total variation distance. Our $(k,p)$-poisoning attacks include the so called ``$p$-tampering'' attacks of \citep{Mahloujifar2018:ALT} as special case by letting $k=m$ ($m$ is the number of parties).  Moreover, $(k,p)$-attacks also include the standard model of $k$ static corruption in secure multi-party computation (in cryptography) letting $p=1$. Our main result in this works is to  prove the \emph{universal} power of $(k,p)$-poisoning as follow. We show that in \emph{any}  $m$-party learning protocol, there exist a $(k,p)$-poisoning adversary that increases probability of the produced hypothesis $h$ having a bad property $B$ (e.g., failing on a particular target instance known to the adversary). 
For the formal  version of Theorem~\ref{thm:mainMPP-Inf}, see Theorem~\ref{thm:mainMPP}.

\begin{theorem}[Power of $(k,p)$-poisoning attacks--\textbf{informal}]  \label{thm:mainMPP-Inf}
Let $\Pi=(P_1,\dots,P_m)$ be an $m$-party learning protocol for an $m$-party learning problem. Also let $B$ be a bad property defined over the output of the protocol. There is a polynomial time $(k,p)$-poisoning attack   such that, given oracle access to the data distribution of the parties, it can increase the probability of $B$ from $\mu$ to $\mu^{1-{kp}/{m}}$. 
\end{theorem}

\paragraph{Example.} By corrupting half of the parties  (i.e., $p=1,k=m/2$) the adversary can increase the probability of any bad event $B$ from $1/100$ to $1/10$.

\paragraph{Universal nature of our attack.} Our attacks are \emph{universal} in the sense that they could be applied to \emph{any} learning algorithm for \emph{any} learning task, and they are \emph{dimension-independent} as they applied to any data distribution. On the other hand, our universal attacks  rely on an initial vulnerability of arbitrary small \emph{constant} probability that is then amplified through the poisoning attack. As a result, although recent poisoning attacks (e.g., see~\citep{koh2018stronger}) obtain \emph{stronger} bounds in their attack against specific defenses,  our attacks apply to \emph{any} algorithm with any built in defenses.

\paragraph{Deriving attacks on federated learning as special case.} Since we allow the distribution of each party in the multi-party case to be completely dependent on that party, our attacks cover the case of model poisoning in federated learning \citep{bagdasaryan2018backdoor,bhagoji2018analyzing},  in which each party sends something \emph{other} than their plain share of data, as special case. 
In fact, multiple works have already demonstrated the power of poisoning attacks and defences in the federated learning setting (e.g., see \citep{fung2018mitigating,bhagoji2018analyzing,chen2018draco,chen2017distributed,guerraoui2018hidden,yin2018byzantine, tomsett2019model, cirincione2019federated, han2019robust}). Some of these attacks obtain stronger quantitative bounds in their attacks, however this is anticipated as these works investigate attacks on  \emph{specific} learners, while a crucial property of our attack is that  our attacks come with \emph{provable} bounds and are \emph{universal} in that they apply to \emph{any} learning task and \emph{any} hypothesis class (including neural nets as special case), if there is an initial $\Omega(1)$ vulnerability (for some bad property) over the generated hypothesis. 

Note that, our attacks actually do not need the exact history of examples that are used by parties, and only need to know the updates sent by the parties during the course of protocol. Suppose an uncorrected party randomizes its local model (e.g., for differential privacy purposes) and shares an update $u_i$ with the server. Knowledge of $u_i$ is enough for our attacker. One might go even further and ask what if the updates are sent in a secure/private way? Interestingly, our attack work in that model too as it only needs to know the effect of the updates on the central model at the end of round $i-1$ (because all the attack wants is  a random continuation of the intermediate model). 

It also worth mentioning that our attack requires sampling oracles from distributions of all the parties. This might seem that we are giving the adversary too much power. However, we think the right way to define security of federated learning is by giving the adversary everything that hat might be leaked to them. This way of defining security is inspired by cryptography. For instance, when modeling the “chosen plaintext” security of encryption schemes, adversary is given access to an encryption oracle, while one might question how realistic it is. Analogously, In federated learning, the adversary can potentially gather some statistics about the distribution of other parties and learn them over time. However, as mentioned above, we do not need to give adversary access to the actual data of honest parties. Only the public effect of them on the shared model is needed.


\paragraph{Further Related Work.} Recent breakthroughs of Diakonikolas et al.~\citep{diakonikolas2016robust} and Lai et al.~\citep{lai2016agnostic} demonstrated the surprising power of algorithmic robust learning over poisoned training data  with limited risk that does not depend on the dimension of the distribution (but still depends on the fraction of poisoned data). These works led to an active line of work (e.g., see \citep{charikar2017learning,diakonikolas2017statistical,diakonikolas2018list,diakonikolas2018sever,prasad2018robust,diakonikolas2018efficient,steinhardt2017certified} and references therein) exploring the possibility of robust statistics over poisoned data with algorithmic guarantees. The works of \citep{charikar2017learning,diakonikolas2018list}, followed by \citep{balcan2008discriminative}, performed \emph{list-decodable} learning, and \citep{balakrishnan2017computationally,diakonikolas2018sever,prasad2018robust} studied supervised learning.

On the negative side, Mahloujifar, Mahmoody and Diochnos \cite{pTampTCC17,mahloujifar2018learning} studied (universal) poisoning attacks that apply to any learning task and any hypothesis class and showed that such attacks can indeed increase the error of any classifiers for any learning problem by a constant probability, so long as there is an initial constant error probability. The attack model used in \cite{pTampTCC17,mahloujifar2018learning}, called $p$-tampering, was a generalization of a similar model introduced in Austrin et al. \cite{austrin2014impossibility} in the bitwise setting in a cryptographic context. These attacks (like the ones in our work) were \emph{universal} in the sense that they could be applied to \emph{any} learning algorithm for \emph{any} learning task, and \emph{dimension-independent} as they applied to any data distribution. On the other hand, these universal attacks  rely on an initial vulnerability of arbitrary small \emph{constant} probability that is then amplified through the poisoning attack. That is why such universal attacks (including those in the multi-party setting) are not in contradiction with the above results.

\subsection{Technical Overview}

Previous universal poisoning attacks of \cite{pTampTCC17,mahloujifar2018learning} for the \emph{single} party case are designed in a setting in which each training example is chosen by the adversary with \emph{independent} probability $p$. We first describe where exactly the ideas of these works come short of extending to the multiparty case, and then we explain how to borrow ideas from attacks on coin-tossing protocols  in cryptography \cite{ben1989collective,haitner2014coin} and obtain the desired attacks of this work.

\paragraph{$p$-tampering attacks and their shortcoming.} For starters, let us assume that the adversary gets to corrupt and control $k$ \emph{randomly} selected parties. In this case, it is easy to see that, at the end every single message in the protocol $\Pi$ between the parties $P_1,\dots,P_m$ is controlled with exactly probability $p=k/m$ by the adversary $\Adv$ (even though these probabilities are correlated). Thus, at a high level it seems  that we should be able to use the $p$-tampering attacks of \cite{pTampTCC17,mahloujifar2018learning} to degrade the quality of the produced hypothesis. However, the catch is that the proof of $p$-tampering attacks of \cite{pTampTCC17,mahloujifar2018learning} (and the bitwise version of \cite{Austrin2017}) crucially rely on the assumption that each message (which in our context corresponds to a training example) is tamperable with \emph{independent} probability $p$, while  corrupting $k$ random parties, leads to tamperable messages in a correlated way.

We prove our main results by first proving  a general result about the power of ``biasing'' adversaries whose goal is to increase the expected value of a random process by controlling each incoming ``segment'' (aka block) of the random process with probability $q$ (think of $q$ as $\approx p \cdot k/m$). These segments/blocks correspond to single or multiple training examples shared during the learning. As these biasing attacks generalize $p$-tampering attacks, we simply call them  \emph{generalized} $p$-tampering attacks. We now describe this attack model and clarify how it can be used to obtain Theorem~\ref{thm:mainMPP-Inf}.

\paragraph{Generalized $p$-tampering: new model for biasing  attacks.} In this work we introduce generalized $p$-tampering (biasing) attacks that are defined for any random process $\xVecVar \equiv (\xVar_1,\dots,\xVar_n)$ and a function $f(\xVecVar) \in [0,1]$ defined over this process. In order to explain the attack model, first consider the setting where there is no attacker. Now, given a prefix $x_1,\dots,x_{i-1}$ of the blocks, the next block $x_i$ is simply sampled  from its conditional probability  distribution $(\xVar_i \mid x_1,\dots,x_{i-1})$. (Looking ahead, think of $x_i$ as the $i$'th training example  provided by one of the parties in the interactive learning protocol.) Now, imagine an adversary who enters the game and whose goal is to increase the expected value of a function $f(\xVar_1,\dots,\xVar_n)$ defined over the random process $\xVecVar$ by tampering with the block-by-block sampling process of $\xVecVar$ described above. Before the attack starts, there will be a a list $S \se [n]$ of ``tamperable'' blocks that is \emph{not} necessarily known to the $\Adv$ in advance, but will become clear to him as the game goes on. Indeed, this set $S$ itself will be first sampled according to some fixed distribution $\SVar$, and the crucial condition we require is that $\Pr[i \in \SVar] = p$ holds for all $i \in [n]$. After $S \gets \SVar$ is sampled,  the sequence of blocks $(x_1,\dots,x_n)$ will be sampled block-by-block as follows. Assuming (inductively) that $x_1,\dots,x_{i-1}$ are already sampled so far, if $i \in S$, then $\Adv$ gets to fully control $x_i$ and determine its value, but if $i \not \in S$, then $x_i$ is simply sampled from its original conditional  distribution $(\xVar_i \mid x_1,\dots,x_{i-1})$. At the end, the function $f$ is computed over the (adversarially) sampled sequence.


We now explain the intuitive connection between generalized $p$-tampering attacks and $(k,p)$-poisoning attacks. The main idea is that we will use a generalized $q$-tampering attack for $q=p\cdot k/m$  over the random process that lists the sequence of training data provided by the parties during the protocol. Let $\SVar$ be the distribution over $[n]$ that picks its members through the following algorithm. First choose a set of random parties $\set{Q_1,\dots,Q_k} \se \set{P_1,\dots,P_m}$, and then for each message $x_j$ that belongs to $Q_i$, include the corresponding index $j$ in the final sampled $S \gets \SVar$ with independent probability $p$.  It is easy to see that $\SVar$ eventually picks every message with (marginal) probability $q=p\cdot k/m$, but it is also the case that these inclusions are not independent events. Finally, to use the power of generalized $p$-tampering attacks over the described $\SVar$ and the random process of messages coming from the parties to get the results of Theorem~\ref{thm:mainMPP-Inf},  roughly speaking, we let a function $f$ model the loss function applied over the produced hypothesis.  Therefore, to prove Theorem~\ref{thm:mainMPP-Inf} it is sufficient to prove Theorem~\ref{thm:mainInf} below which focuses on the power of generalized $p$-tampering biasing attacks.

\begin{theorem}[Power of generalized $p$-tampering-\textbf{informal}] \label{thm:mainInf}
Suppose $\xVecVar \equiv  (\xVar_1,\dots, \xVar_n)$ is a joint distribution such that, given any prefix, the remaining blocks could be efficiently sampled in polynomial time. Also let $f \colon \Supp(\xVecVar) \To [0,1]$. Then, for any set distribution $\SVar$ for which $\Pr[i \in \SVar]=p$ for all $i$, there is a polynomial-time generalized $p$-tampering attack (over tampered blocks in $\SVar$) that increases the average of $f$ over its input from $\mu$ to (arbitrarily close to) $\mu' = \mu^{-p}\cdot \Ex[f(\xVec)^{1+p}]$. In particular, if $f$ is boolean function, then it hold that $\mu'= \mu^{1-p}$.
\end{theorem}
 The formal statement of Theorem~\ref{thm:mainInf} above follows from Theorem~\ref{thm:main}.

\paragraph{Bitwise vs. blockwise attacks.} It is easy to see that in the definition of generalized $p$-tampering attacks, it does not matter whether we define the attack bit-by-bit or block-by-block. The reason is that, even if we break down each block into smaller bits, then still each bit shall eventually fall into the set of tamperable bits, and the model allows correlation between the inclusion and exclusion of each block/bit into the final tamperable set. This is in contrast to the $p$-tampering model for which this equivalence is not true. In fact, optimal bounds achievable by bitwise $p$-tampering as proved in \cite{Austrin2017} are \emph{impossible} to achieve in the blockwise $p$-tampering setting \cite{pTampTCC17}. Despite this simplification, we still prefer to use a blockwise presentation of the random process, as this way of modeling the problem allows better tracking measures for  the attacker's sample complexity.

\paragraph{Ideas behind the tampering attack of Theorem \ref{thm:mainInf}.} To prove Theorem \ref{thm:mainInf} we use ideas from  \cite{haitner2014coin,ben1989collective} in the context of coin-tossing attacks and generalize them  using new  techniques to obtain our generalized $p$-tampering attacks.

{\em Rejection sampling attack.} The simplified version of our attack can be described as follows. Based on the nature of this attack, we call it the ``rejection sampling'' (RS) attack.
For any prefix of already sampled blocks $(x_1,\dots,x_{i-1})$, suppose the adversary is given the chance of controlling the next $i$'th block. 
The RS tampering then works as follows:
\begin{enumerate}
\item \label{step:oneone} Let $x'_i,\dots,x'_n$ be a random continuation of the random process, conditioned on $(x_1,\dots,x_{i-1})$.

\item If  $s=f(x_1,\dots,x_{i-1},x'_i,\dots,x'_n)$, then if $s=1$  output $y_i$, and otherwise (i.e., if $s=0$) go to Step \ref{step:oneone} and repeat the sampling process.
\end{enumerate}

The above attack is inspired by the two-party attack of \cite{haitner2014coin}.
Our main contribution is to do the following steps. (1) First, analyze  this attack in the generalized tampering setting and show its power, which implies the multiparty case as special case. This already gives an alternative, and in our eyes simpler, proof of the classic result of \cite{ben1989collective}  (2) We then extend this attack and its analysis to the \emph{real-output} setting. (3) Finally, we show how to  approximate  this attack in polynomial time. 

\section{Multi-Party Poisoning: Definitions and Main Results}

\paragraph{Notation.} We use bold font  (e.g., $\xVar, \SVar, \alphaVar$)  to represent random variables, and usually use same non-bold letters for denoting samples from these distributions. We use $d \gets \dist$ to denote the process of sampling $d$ from the random variable $\dist$. By  $\Ex[\alphaVar]$ we mean the expected value of $\alphaVar$ over the randomness of $\alphaVar$, and by $\Var[\alphaVar]$ we denote the variance of random variable $\alphaVar$. We might use a ``processed'' version of $\alphaVar$, and use $\Ex[f(\alphaVar)]$ and $\Var[f(\alphaVar)]$ to denote the expected value and variance, respectively, of $f(\alphaVar)$ over the randomness of $\alphaVar$. 
A learning problem $(\cA,\cB,\dist,\H)$ is specified by the following components. The set $\cA$ is the set  of possible \emph{instances},
$\cB$ is the set of possible \emph{labels}, 
$\dist$ is distribution over $\cA \times \cB$.\footnote{By using joint distributions over $\cA \times \cB$, we jointly model a set of distributions over $\cA$ and a concept class mapping $\cA$ to $\cB$ (perhaps with noise and uncertainty).}
The set 
$\H \se \cB^\cA$
is called the \emph{hypothesis space} or \emph{hypothesis class}.
An \emph{example} $s$ is a pair $s=(a,b)$ where $x \in \cA$ and $y \in \cB$. 
We  consider \emph{loss functions} $\Loss \colon \cB \times \cB \To \Rplus$ where $\Loss(b',b)$ measures how different the `prediction' $y'$ (of some possible hypothesis $h(a)=y'$) is from the true outcome $y$.
We call a loss function \emph{bounded} if it always takes values  in $[0,1]$. 
A natural loss function for classification tasks is to use $\Loss(b',b)=0$ if $y=y'$ and $\Loss(b',b)=1$ otherwise.
The \emph{risk} of a hypothesis $h \in \C$ is the expected loss of $h$ with respect to $\dist$, namely $\Risk(h) = \Ex_{(a,b)\gets \dist}[\Loss(h(a),b)]$. The \emph{average error} which quantifies the total error of the protocol is defined as $\error(\dist) = \Pr_{h\gets \Pi, (a,b)\gets \dist}[\Loss(h(a), b)].$

\begin{definition}[Multi-party learning protocols]
An $m$-party learning protocol $\Pi$ for the $m$-party learning problem $(\D,\H)$ consists of an aggregator function $G$ and $m$ (interactive) data providers $\cP=\set{P_1,\dots,P_m}$. For each data provider $P_i$, there is a distribution $\dist_i\in \D$ that models the (honest) distribution of labeled samples generated by $P_i$, and there is a final (test) distribution $\dist$ that $\cP,G$ want to learn jointly. The protocol runs in $r$ rounds and at each round, based on the protocol $\Pi$, one particular data owner $P_i$ broadcasts a single labeled example $(a,b)\gets \dist_i$.\footnote{We can directly model settings where more data is exchanged in one round, however, we stick to the simpler definition w.l.o.g.} In the last round, the aggregator function $G$ maps the the messages to an output hypothesis $h\in \H$. 
\end{definition}

Now, we define poisoning attackers that target multi-party protocols. We formalize a more general notion that includes $p$-tampering attacks and  $k$-party corruption as special case.

\begin{definition}[Multi-party $(k,p)$-poisoning attacks] \label{def:kp-poison}
A $(k,p)$-poisoning attack against an $m$-party learning protocol $\Pi$ is defined by an adversary $\Adv$ who can control a subset $\C\subseteq [m]$ of the parties where $|\C| = k$. The attacker $\Adv$ shall pick the set $\C$ at the beginning. At each round $j$ of the protocol, if a data provider $P_i\in\C$ is supposed to broadcast the next example from its distribution $\dist_i$, the adversary can partially control this sample using the tampered distribution $\mal{\dist}$ such that $|\mal{\dist} - \dist_i| \leq p$ in total variation distance. Note that the distribution $\mal{\dist}$ can depend on the history of examples broadcast so far, but the requirement is that, conditioned on this history, the malicious message of adversary modeled by distribution $\mal{\dist}$, is at most $p$-statistically far from $\dist_i$.  We use $\Pi_\Adv$ to denote the protocol in presence of $\Adv$. We also define the following notions.
 $\Adv$ is a \emph{plausible} adversary, if it always holds that $\Supp(\mal{\dist}) \se \Supp(\dist_i)$.
 $\Adv$ is \emph{efficient} if it runs in polynomial time in the total length of the messages exchanged during the protocol (from the beginning till end).
\end{definition}

\begin{remark}[Static vs. adaptive corruption]  Definition~\ref{def:kp-poison} focuses on corrupting $k$ parties statically. A natural extension of this definition  in which the set $\C$ is chosen \emph{adaptively} \cite{canetti1996adaptively} while the protocol is being executed can also be defined naturally. However, here we focus on static corruption  and leave the possibility of improving our results in the adaptive case for future work.
\end{remark}
\begin{remark}[Plausible vs. clean-label attacks] In recent years, the term clean-label is used for data-poisoning attacks that must use the correct label for the poison data. This definition is special case of plausibility as one can define the support set to be the all images with their correct label. However, our definition is more general and can be used for model poisoning attacks against federated learning as well. For example, the distribution $\dist$ could be the distribution of gradients at a certain round. Being plausible in this setting means that the gradient must be calculated based on a real input and cannot be arbitrary.  

\end{remark}
We now formally state our result about the power of $(k,p)$-poisoning attacks.

\begin{theorem}[Power of efficient multi-party poisoning] \label{thm:mainMPP} In any $m$-party protocol $\Pi$ for parties $\cP=\set{P_1,\dots,P_m}$, for any $p\in[0,1]$ and $k\in[m]$, the following hold where $M$ is the total length of the messages exchanged.
\begin{enumerate}
    \item For any bad property $B: \cH\to \set{0,1}$, there is a \emph{plausible} $(k,p)$-poisoning attack $\Adv$ that runs in time $\poly(M/\eps)$ and increases the probability of $B$ from  $\mu$ (in the no-attack setting) to 
    $$\mu'\geq \mu^{1-p} - \eps.$$
    \item If the (normalized) loss function is bounded (i.e., it outputs in $[0,1]$), then there is a \emph{plausible}, $(k,p)$-poisoning $\Adv$ that runs in time $\poly(M/\eps)$ and increases the average error of the protocol as
    \begin{align*}
    \error_\Adv(\dist) &\geq \error(\dist)^{-p} \cdot \Ex_{h\gets \Pi}[\Risk(h,\dist)^{1+p}]\\& \geq \error(\dist) + \frac{p\cdot k}{2m}\cdot \nu -\eps
    \end{align*}
    where $\nu=\Var_{h\gets \Pi}[\Risk(h,\dist)]$ and $\Var[\cdot]$ is the variance.
\end{enumerate}

\end{theorem}

\paragraph{Allowing different distributions in different rounds.} In Definition~\ref{def:kp-poison}, we restrict the adversary to remain ``close'' to $\dist_i$ for each message sent out by one of the corrupted parties. A natural question is: what happens if we allow the parties distributions to be different in different rounds. For example, in a round $j$, a party $P_i$ might send \emph{multiple} training examples $D^{(j)}= \big(d^{(j)}_1,d^{(j)}_2,\dots,d^{(j)}_k\big)$, and we want to limit the \emph{total} statistical distance between the distribution of the larger message  $D^{(j)}$ from $\dist_i^k$ (i.e., $k$ iid samples from $\dist_i$).\footnote{Note that, even if each block in $\big(d^{(j)}_1,d^{(j)}_2,\dots,d^{(j)}_k\big)$ remains $p$-close to $\dist_i$, their joint distribution could be quite far from $\dist_i^k$.} We emphasize that, our results extend to this more general setting as well. In particular, the proof of Theorem~\ref{thm:mainMPP} directly extends to a more general setting where we can allow the honest distribution $\dist_i$ of each party $i$ to also depend on the round $j$ in which these messages are sent. Thus,  we can use a round-specific distribution $\dist_i^{(j)}$ to model the joint distribution of \emph{multiple} samples $D^{(j)}=\big(d^{(j)}_1,d^{(j)}_2,\dots,d^{(j)}_k\big)$ that are sent out in the $j$'th round by the party $P_i$. This way, we can obtain the stronger form of attacks that remain statistically close to the joint (correct) distribution of the (multi-sample) messages   sent in a round. In fact, as we will discuss shortly $D^{(j)}$ might be of completely different type.

\paragraph{Allowing randomized aggregation.}
The aggregator $G$ is a simple function that maps the transcript of the exchanged messages to a hypothesis $h$. A natural question is: what happens if we generalize this to the setting where $G$ is allowed to be randomized. We note that in Theorem~\ref{thm:mainMPP}, Part 2 can allow $G$ to be randomized, but Parts 1 and 3 need deterministic aggregation. The reason is that for those parts, we need the transcript to determine the confidence and average error functions. One general way to make up for randomized aggregation is  to allow the parties to inject randomness into the transcript as they run the protocol by sending messages that are not necessarily learning samples from their distribution $\dist_i$. As described above, our attacks extend to this more general setting as well.  Otherwise, we will need the adversary to be able to also depend on the randomness of $G$, but that is also a reasonable assumption if the aggregation is used using public beacon that could be obtained by the adversary as well.

Before proving Theorem~\ref{thm:mainMPP}, we need to develop our main result about the power of generalized $p$-tampering attacks. In Section~\ref{sec:GenP}, we prove such result, and then in Section~\ref{sec:proofMPP} we prove Theorem~\ref{thm:mainMPP}.
\section{Multi-Party Poisoning via Generalized $p$-Tampering} \label{sec:GenP}

To prove our Theorem \ref{thm:mainMPP} we interpret the multi-party learning protocol as a coin tossing protocol in which the final bit is $1$ if $h$ has the (bad) property  $B$.
We define a corresponding attack model in coin tossing protocols that can be directly used to obtain the desired goal; this model is called \emph{generalized} $p$-tampering.
Below, we formally state our main result about the power of generalized $p$-tampering attacks. We start by formalizing some notation and  definitions.


\paragraph{Notation.} By $\xVar \equiv \yVar$ we denote that the random variables $\xVar$ and $\yVar$ have the same distributions. Unless stated otherwise, by using a bar over a variable, we emphasize that it is a vector.  By $\xVecVar \equiv (\xVar_1,\xVar_2,\dots,\xVar_n)$ we refer to a joint distribution over vectors with $n$ components.
For a joint distribution $\xVecVar \equiv (\xVar_1,\dots,\xVar_n)$, we use $\pfix{\xVar}{i}$ to denote the joint distribution of the first $i$ variables $\xVecVar \equiv (\xVar_1,\dots,\xVar_i)$. Also, for a vector $\xVec=(x_1\dots x_n)$ we use $\pfix{x}{i}$ to denote the prefix $(x_1,\dots, x_i)$.
For a randomized algorithm $L(\cdot)$, by $y \gets L(x)$ we denote the randomized execution of $L$ on input $x$ outputting $y$. For a distribution $(\xDist,\yDist)$, by $(\xDist \mid \yDist)$ we denote the conditional distribution $(\xDist \mid \yDist = y)$. By $\Supp(\dist) = \set{d \mid \Pr[\dist=d]>0}$ we denote the support set of $\dist$. 
By $T^\dist(\cdot)$ we denote an algorithm $T(\cdot)$ with oracle access to a sampler for $\dist$ that upon every query returns fresh samples from $\dist$. 
By $\dist^n$ we denote the distribution that returns $n$  iid samples from $\dist$.

\begin{definition}[Valid prefixes]
Let $\xVecVar \equiv (\xVar_1,\dots,\xVar_n)$ be an arbitrary joint distribution. We call $\pfix{x}{i} = (x_1,\dots, x_i)$ a \emph{valid prefix} for $\xVecVar$ if there exist  $x_{i+1},\dots, x_n$ such that $(x_1,\dots, x_n) \in \Supp (\xVecVar)$. $\VP(\xVecVar)$ denotes the set of all valid prefixes of $\xVecVar$.
\end{definition}

\begin{definition} [Tampering with random processes] \label{def:tamp} Let $\xVecVar \equiv (\xVar_1,\dots,\xVar_n)$ be an arbitrary joint distribution.   We call a (potentially randomized and possibly computationally unbounded) algorithm $\Tamp$ an (online) \emph{tampering algorithm for} $\xVecVar$ if given any prefix  $\pfix{x}{i-1} \in \VP(\xVecVar)$, we have
$$\Pr_{x_i \gets \Tamp(\pfix{x}{i-1})} [\pfix{x}{i} \in \VP(\xVecVar)]=1 ~.$$
Namely, $ \Tamp(\pfix{x}{i-1})$  outputs  $x_i$ such that  $\pfix{x}{i}$ is again a valid prefix. We call $\Tamp$ an \emph{efficient} tampering algorithm for $\xVecVar$ if it runs in time $\poly(N)$ where $N$ is the  bit length of $\xVec \in \Supp(\xVecVar)$.
\end{definition}

\begin{definition}[Online samplers]
We call $\Samp$ an \emph{online sampler} for $\xVecVar \equiv (\xVar_1,\dots,\xVar_n)$ if for all $\pfix{x}{i-1} \in \VP(\xVecVar)$, $\Samp(n,\pfix{x}{i-1}) \equiv \xVar_i$. Moreover, we call $\xVecVar \equiv (\xVar_1,\dots,\xVar_n)$ \emph{online samplable} if it has an online sampler that runs in time $\poly(N)$ where $N$ is  the  bit length of   $\xVec \in \Supp(\xVecVar)$.
\end{definition}

\paragraph{Notation for tampering distributions.}
Let $\xVecVar \equiv (\xVar_1,\dots,\xVar_n)$ be an arbitrary joint distribution and $\Tamp$ a tampering algorithm for $\xVecVar$. For any subset $S \se [n]$, we define $\yVecVar\equiv \twist{\xVecVar}{\Tamp}{S}$ to be the joint distribution that is the result of online tampering of $\Tamp$ over set $S$, where $\yVecVar \equiv (\yVar_1,\dots,\yVar_n)$ is sampled  inductively  as follows. For every $i\in[n]$, suppose $ \pfix{y}{i-1}$ is the previously sampled block. If $i \in S$, then the $i\th$ block $\yVar_i$ is generated by the tampering algorithm $\Tamp( \pfix{y}{i-1})$, and otherwise,  $\yVar_i$ is sampled from  $(\xVar_i \mid \xVar_{i-1} =  \pfix{y}{i-1})$. For any \emph{distribution} $\SVar$ over subsets of $[n]$, by $\twist{\xVecVar}{\Tamp}{\SVar}$ we denote the random variable that can be sampled by first sampling $S \gets \SVar$ and then   $\yVec \gets \twist{\xVecVar}{\Tamp}{S}$.


\begin{definition}[$p$-covering]
Let $\SVar$ be a distribution over the subsets of $[n]$. We call $\SVar$ a \emph{$p$-covering} distribution on $[n]$ (or simply $p$-covering, when $n$ is clear from the context), if for all $i \in [n], \Pr_{S \gets \SVar}[ i \in S] = p$.
\end{definition}

The following theorem states the power of generalized $p$-tampering attacks.

\begin{theorem}[Biasing of bounded functions through generalizing $p$-tampering]  \label{thm:main}
Let $\SVar$ be a $p$-covering  distribution on $[n]$,  $\xVecVar \equiv  (\xVar_1,\dots, \xVar_n)$ be a joint distribution, $f \colon \Supp(\xVecVar) \To [0,1]$, and $\mu = \Ex[f(\xVecVar)]$. 
Then,  
for any $\eps \in [0,1]$, there exists a tampering algorithm $\Tamp_{\eps}$ that, given oracle access to $f$ and any online sampler $\Samp$ for $\xVecVar$, it  runs in time $\poly({N}/{ \eps})$, where $N$ is the bit length of any $\xVec \gets \xVecVar$, and for $\yVecVar_\eps \equiv \twist{\xVecVar}{\Tamp_\eps^{f,\Samp}}{\SVar}$, 
    it holds that
    $$\Ex\left[f(\yVecVar_\eps)\right] \geq 
    \mu^{-p} \cdot \Ex\left[f(\xVecVar)^{1+p}\right] - \eps ~.$$
\end{theorem}

\paragraph{Special case of Boolean functions.} When the function $f$ is Boolean, we get $\mu^{-p} \cdot \Ex[f(\xVecVar)^{1+p}] = \mu^{1-p} \geq \mu (1 + \Omega_\mu(p))$, which matches the bound proved in \cite{ben1989collective} for the special case of $p=k/n$ for integer $k \in [n]$ and for $\SVar$ that is  uniformly random subset of $[n]$ of size $k$. (The same bound for the case of 2 parties was proved in  \cite{haitner2014coin} with extra properties). Even for this case, compared to \cite{ben1989collective,haitner2014coin} our result is more general, as we can allow $\SVar$ with arbitrary $p \in [0,1]$ \emph{and} achieve a polynomial time attack given oracle access to an online sampler for $\xVecVar$. The work of \cite{haitner2014coin} also deals with polynomial time attackers for the special case of 2 parties, but their efficient attackers use a different oracle (i.e., OWF inverter), and it is not clear whether or not their attack extend to the case of more then 2 parties. Finally, both \cite{ben1989collective,haitner2014coin} prove their bound for the \emph{geometric} mean of the averages for different $S \gets \SVar$, while we do so for their arithmetic mean, but we emphasize that this is enough for all of our applications.
 
 The  bounds of Theorem \ref{thm:main} relies on the quantity $\mu'=\mu^{-p}\cdot \Ex[f(\xVecVar)^{1+p}]$. A natural question is: how large is $\mu'$ compared to $\mu$? As discussed above, for the case of Boolean  $f$, we already know that $\mu' \geq \mu$, but that argument does not apply to the real-output $f$. A simple application of Jensen's inequality shows that $\mu \leq \mu'$ in general, but that still does not mean that $\mu' \gg \mu$.
 

\paragraph{General case of real-output functions: relating the  bias to the variance.} 
If $\Var[f(\xVecVar)]=0$, then no tampering attack can achieve any bias, so any the minimum bias of all attacks should somehow depend on the variance of $f(\xVecVar)$. In the following, we show that this gap does exist and that $\mu' - \mu \geq \Omega(p \cdot \Var[f(\xVecVar)])$. Similar results relating the bias the the variance of the original distribution were previously proved \citep{mahloujifar2018learning,pTampTCC17,austrin2014impossibility}  for the special case of $p$-tampering attacks   (i.e., $\SVar$ chooses every $i \in [n]$ independently with probability $p$). Here, we obtain a more general result for any $p$-covering set structure  $\SVar$.

\begin{corollary}\label{cor:biasToVar}
If $\nu=\Var[f(\xVecVar)]$, then the computationally bounded attack of Theorem \ref{thm:main} achieves  
    $$\Ex\left[f(\yVecVar_\eps)\right]   - \Ex\left[f(\xVecVar)\right] 
    \geq \frac{p\cdot (p+1)}{2\cdot \mu^p} \cdot \nu   - \eps\geq  \frac{p}{2} \cdot \nu  - \eps ~.$$
\end{corollary}
We prove this corollary in Appendix \ref{appendix:variance} by proving a connection between the bound of Theorem \ref{thm:main} and the variance $\nu$.

\subsection{Proving Theorem \ref{thm:main}}
Here, we first prove the power of computationally unbounded adversaries. Then, we show how we can ``approximate'' this attack with a polynomial-time adversary and get almost the same bias.
\subsubsection{Warm up: Computationally Unbounded Adversaries}

The construction below describes a computationally unbounded biasing algorithm that achieves the bounds of Theorem \ref{thm:main}.
\begin{construction}[Rejection-sampling tampering]\label{const:RS}
Let $\xVecVar \equiv (\xVar_1,\dots, \xVar_n)$  and $f \colon \Supp(\xVecVar) \To [0,1]$. The \emph{rejection sampling} tampering algorithm $\RTam^f$ works as follows. Given the valid prefix $\pfix{y}{i-1} \in \VP(\xVecVar)$, the tampering algorithm would do the following:
\begin{enumerate}
\item \label{step:one} Sample $y_{\geq i} \gets (\xVar_{\geq i} \mid \pfix{y}{i-1})$ by using the online  sampler for $f$.

\item If  $s=f(y_1,\dots,y_n)$, then with probability $s$ output $y_i$, otherwise go to Step \ref{step:one} and repeat.

\end{enumerate}
\end{construction}

We will first prove a property of the rejection sampling algorithm when applied on every block.

\begin{definition}[Notation for partial expectations of functions]
\label{defs:biasing} Suppose $f \colon \Supp(\xVecVar) \To \R$ is defined over a joint distribution $\xVecVar \equiv (\xVar_1,\dots,\xVar_n)$, $i \in [n]$, and  $\pfix{x}{i}\in \VP(\xVecVar)$. Then, using a small hat, we define the notation $\fhat{x}{{ i}}= \Ex_{\xVec\gets ({\xVecVar} \mid \pfix{x}{i})}[f(\xVec) ]$. E.g., for $\xVec=x_{[n]}$, we have $\fhat{\xVec}{}=f(\xVec)$.  

\end{definition}

\begin{claim} \label{clm:dist}
If $\twist{\xVecVar}{\RTam^f}{ [n]} \equiv \yVecVar^{[n]} \equiv (\yVar_1,\dots,\yVar_n)$. Then, for every valid   $\pfix{y}{i} \in \VP[\xVecVar]$,
$$\frac{\Pr[\pfix{\yVar}{i} = \pfix{y}{i}]}{\Pr[\pfix{\xVar}{i} = \pfix{y}{i}]} = \frac{\fhat{y}{i}}{\mu} ~.$$
\end{claim}

\begin{proof}
Based on the description of $\RTam^f$, for any $\pfix{y}{i} \in \VP(\xVecVar)$ the following equation holds for the probability of sampling $y_i$ conditioned on prefix $\pfix{y}{i-1}$.
\begin{align*}
    \Pr[\yVar_i &= y_i \mid \pfix{y}{i-1}] =  \Pr[\xVar_i = y_i \mid \pfix{y}{i-1}]\cdot \fhat{y}{i}\\
    &~~~~+ (1-\fhat{y}{i-1})\cdot \Pr[\yVar_i = y_i \mid \pfix{y}{i-1}].
\end{align*}
The first term in this equation corresponds to the probability of selecting and accepting in the first round of sampling and the second term corresponds to the probability of selecting and accepting in any round except the first round. Therefore we have
$$\Pr[\yVar_i = y_i \mid \pfix{y}{i-1}] =  \frac{\fhat{y}{i}}{\fhat{y}{i-1}}\cdot\Pr[\xVar_i = y_i \mid \pfix{y}{i-1}] ~,$$
which implies that
\begin{align*}
\Pr[\pfix{\yVar}{i} = \pfix{y}{i}] &=  \prod_{j\in [i]}\left(\frac{\fhat{y}{j}}{\fhat{y}{j-1}}\right)\cdot\Pr[\pfix{\xVar}{i}
= \pfix{y}{i}]\\&=\frac{\fhat{y}{i}}{\mu}\cdot\Pr[\pfix{\xVar}{i} = \pfix{y}{i}] ~.
\end{align*}
\end{proof}

Now, we prove two properties for \emph{any} tampering algorithm (not just rejection sampling) over a $p$-covering  distribution.

\begin{lemma} \label{lem:covering}
Let $\SVar$ be $p$-covering for $[n]$ and $\yVec \in \Supp(\xVecVar)$. For any $S \in \Supp(\SVar)$ and an arbitrary tampering algorithm $\Tam$ for $\xVecVar$, let $\yVecVar^{S} \equiv \twist{\xVecVar}{\Tam}{S}$. Then,
$$\prod_{S\in 2^{[n]}} \left(\frac{\Pr[\yVecVar^S=\yVec] }{\Pr[\xVecVar=\yVec] } \right)^{\Pr[\SVar=S]}  = \left(\frac{\Pr[\yVecVar^{[n]}=\yVec]}{\Pr[\xVecVar=\yVec]}\right)^p ~.$$
\end{lemma}

\begin{proof}

For every $\pfix{y}{i} \in \VP (\yVecVar^{[n]}) \se \VP (\xVecVar)$  define $\rho[\pfix{y}{i}]$ as 
    $$\rho[\pfix{y}{i}] = \frac{\Pr[\yVar^{[n]}_i = x_i \mid \pfix{\yVar^{[n]}}{i-1}=\pfix{y}{i-1} ]}{\Pr[\xVar_i = x_i \mid \pfix{\xVar}{i-1} = \pfix{y}{i-1}]} ~.$$

Then, for all $\yVec \in \VP (\yVecVar^{S}) \se \VP (\xVecVar)$ we have
        $$\Pr[\yVecVar^S = y] = \Pr[\xVecVar =  y] \cdot \prod_{i \in S} \rho[\pfix{y}{i}] ~.$$
Therefore we have 
\begin{align*}
\prod_{S\in 2^{[n]}} \left(\frac{\Pr[\yVecVar^S=\yVec] }{\Pr[\xVecVar=\yVec] } \right)^{\Pr[\SVar=S]} 
=\left(\prod_{i \in [n]} \rho[\pfix{y}{i}]\right)^p.
\end{align*}
\end{proof}

\begin{claim}\label{clm:main} Suppose $\SVar$ is $p$-covering on $[n]$, $\yVecVar^{S} \equiv \twist{\xVecVar}{\Tam}{S}$ for any $S \gets \SVar$, and $ \yVecVar \equiv \twist{\xVecVar}{\Tam}{\SVar}$ for an arbitrary tampering algorithm $\Tam$ for $\xVecVar$. Then, it holds that
$$\Ex [f(\yVecVar)] \geq \sum_{\yVec \in \Supp(\xVecVar)}\Pr[\xVecVar=\yVec]\cdot f(\yVec) \cdot \left(\frac{\Pr[\yVecVar^{[n]}=\yVec]}{\Pr[\xVecVar=\yVec]}\right)^p ~.$$
\end{claim}
\begin{proof}
Let $h_{S,\yVec} = \frac{\Pr[\yVecVar^{S}=\yVec]}{\Pr[\xVecVar=\yVec]}$. Also let $\cZ\se\Supp(\xVecVar)$. Note that $\Supp(\yVecVar^S) \se \cZ$ for any $S \se [n]$. Therefore, we have $\Ex [f(\yVecVar)] = \Ex_{\substack{S \gets \SVar}}  \Ex_{\yVec \gets \yVecVar^{S} } \left[f( y)\right]$ is equal to
\begin{align*}
        &\sum_{S\in 2^{[n]}} \Pr[\SVar=S] \cdot \sum_{\yVec \in \cZ}  \Pr[\yVecVar^S=\yVec] \cdot f(\yVec)\\
         = &\sum_{S\in 2^{[n]}} \Pr[\SVar=S] \cdot \sum_{\yVec \in \cZ} h_{S,\yVec} \cdot \Pr[\xVecVar=\yVec] \cdot f(\yVec)\\
        = &\sum_{\yVec \in \cZ} \Pr[\xVecVar=\yVec]\cdot f(\yVec)\cdot \sum_{S\in 2^{[n]}}\Pr[\SVar=S] \cdot h_{S,\yVec} \\
        &~~~\text{(by AM-GM inequality)}~~ \\
        \geq &\sum_{\yVec \in \cZ} \Pr[\xVecVar=\yVec]\cdot f(\yVec)\cdot \prod_{S\in 2^{[n]}} h_{S,\yVec}^{\Pr[\SVar=S]}\\
         &\text{~~~(by $p$-covering of $\SVar$ and Lemma \ref{lem:covering})}~~ \\
         = &\sum_{\yVec \in \cZ}\Pr[\xVecVar=\yVec]\cdot f(\yVec) \cdot \left(\frac{\Pr[\yVecVar^{[n]}=\yVec]}{\Pr[\xVecVar=\yVec]}\right)^p.\
\end{align*}
\end{proof}

  We now prove the main result using the one-rejection sampling tampering algorithm and also relying on the $p$-covering property of $\SVar$. In particular, if $\yVecVar \equiv \twist{\xVecVar}{\RTam^f}{\SVar}$, then by Claims \ref{clm:main} and \ref{clm:dist} we have
\begin{align*}
    \Ex [f(\yVecVar)] &\geq \sum_{\yVec \in \Supp(\xVecVar)}\left(\frac{\Pr[\yVecVar^{[n]}=\yVec]}{\Pr[\xVecVar=\yVec]}\right)^p\cdot\Pr[\xVecVar=\yVec]\cdot f(\yVec)\\
        &\text{~~~(by Claim \ref{clm:dist})}~~\\
        &= \sum_{\yVec \in \Supp(\xVecVar)}\left(\frac{f(\yVec)}{\mu}\right)^p\cdot\Pr[\xVecVar=\yVec]\cdot f(\yVec)\\
        &= \mu^{-p} \cdot \sum_{\yVec \in \Supp(\xVecVar)}\Pr[\xVecVar=\yVec]\cdot f(\yVec)^{1+p}\\
        &= \mu^{-p} \cdot \Ex[f(\xVecVar)^{1+p}] ~.
\end{align*}

\subsubsection{Proving Theorem \ref{thm:main} for Polynomially Bounded Attacks}

In this section, we prove the second item of Theorem~\ref{thm:main}. Namely, we show an efficient tampering algorithm whose average is $\eps$-close to the average of $\RTam$. We define this attack as follows:

\begin{construction}[$k$-rejection-sampling tampering]\label{const:kRS}
Let $\xVecVar = (\xVar_1,\dots, \xVar_n)$ be a joint distribution and $f \colon \Supp(\xVecVar) \To [0,1]$. The \emph{$k$-rejection sampling} tampering algorithm $\RTamk^f$ works as follows. Given the valid prefix $\pfix{y}{i-1} \in \VP(\xVecVar)$, the tampering algorithm would do the following for $k$ times:
\begin{enumerate}
\item \label{step:one-kRS} Sample $y_{\geq i} \gets (\xVar_{\geq i} \mid \pfix{y}{i-1})$ by using the online  sampler for $f$.

\item Let  $s=f(y_1,\dots,y_n)$; with probability $s$ output $y_i$, otherwise go to Step \ref{step:one-kRS}.

\end{enumerate}

\noindent If no $y_i$ was output during any of the above $k$ iterations then output a fresh sample $y_i \gets (\xVar_i \mid \pfix{y}{i-1})$.
\end{construction}

The output distribution of $\RTamk$ on any input,  converges to the rejections sampling tampering algorithm $\RTam$ for sufficiently large  $k \to \infty$.

{\bf  Notation.} Below, use the notation $\zVecVar = \twist{\xVecVar}{\RTamk^f}{S}$ and $\mu_k = \Ex[f(\zVecVar)]$.

We will prove the following claim which will directly completes the proof of second part of Theorem~\ref{thm:main}. 

\begin{claim}
\label{clm:effAttack}
Let $\xVecVar = (\xVar_1,\dots, \xVar_n)$ be a joint distribution and $f \colon \Supp(\xVecVar) \To [0,1]$. For any $\eps \in [0,1]$, let $k \geq \frac{16\ln (2n /\eps)}{\eps^2\mu^2}$. Then $\RTamk$ runs in time $O(k) = \poly(N/(\eps\cdot \mu))$, where $N \geq n$ is the total bit-length of representing $\xVecVar$, and for $\zVecVar \equiv \twist{\xVecVar}{\RTamk^{f,\Samp}}{\SVar}$ it holds that
$$\Ex[f(\zVecVar)] \geq \mu^{-p} \cdot \Ex[f(\xVecVar)^{1+p}] - \eps ~.$$
\end{claim}
\begin{proof}
It is easy to see why $\RTamk$ runs in time $O(k)$ and thus we will focus on proving the expected value of the output of the $k$-rejection sampling tampering algorithm. To that end, we start by providing some definitions relevant to our analysis.

\remove{

\begin{claim} \label{clm:muk}
$\mu_k \geq \mu_k - n \cdot (1-\delta)^k$.
\end{claim}

\begin{proof}
For the same prefix $\pfix{y}{i-1}$, if $\fhat{y}{i-1} < \delta$, then $\RTamk$ continues to produce something non-negative, while $\RTamd$ produces output $0$. On the other hand, if $\fhat{y}{i-1} < \delta$, the probability of $\RTamk$ \emph{not} succeeding even after $k$ resetting is at most $(1-\delta)^k$. By a union bound, this happens at \emph{some} step during the sampling of $\zVecVar$ with probability at most $(1-\delta)^k$.
\end{proof}

Now, we analyze $\mu_k$ and try to compare it with the average under rejection sampling.
}

\begin{definition}
For $\delta \geq 0$, let
$$\High(\delta) = \set{\xVec \mid \xVec \in \Supp(\xVecVar) \land \forall i\in [n], \fhat{x}{i-1} \geq \delta  },~
\Low(\delta) = \Supp(\xVecVar) \sm \High(\delta)~, $$
$$\BigOut(\delta) = \set{\xVec \mid \xVec \in \Supp(\xVecVar) \land f(\xVec) \geq \delta}, \text{~and~~} \SmallOut(\delta) = \Supp(\xVecVar) \sm \BigOut(\delta) ~.$$
\end{definition}

\begin{claim} \label{clm:LowBig}
For $\delta_1 \cdot \delta_2 = \delta$, it holds that
$$\Pr_{\xVec \gets \xVecVar}[\xVec \in \BigOut(\delta_1) \mid \xVec \in \Low(\delta)] \leq \delta_2 ~.$$
As a result, it holds that $\Pr_{\xVec \gets \xVecVar}[\xVec \in \BigOut(\delta_1) \land \xVec \in \Low(\delta)] \leq \delta_2$, and so
$$\sum_{\xVec \in  \BigOut(\delta_1) \cap \Low(\delta)} \Pr[\xVec = \xVecVar] \leq \delta_2 ~.$$
\end{claim}
\begin{proof}
Let $t\colon \Low(\delta) \to \VP({\xVecVar})$ be such that $t(\xVec)$ is the smallest prefix $\pfix{x}{i}$ such that $\fhat{x}{i}\leq \delta$. Now consider the set $T=\{t(\xVec) \mid \xVec \in \Low(\delta)\}$. For any $w\in T$ we have $$\delta \geq \fhat{w}{} \geq \Pr_{\xVec \gets \xVecVar}[\xVec \in \BigOut(\delta_1)\mid t(\xVec) = w]\cdot \delta_1 ~,$$
which implies
$$\Pr_{\xVec \gets \xVecVar}[\xVec \in \BigOut(\delta_1)\mid t(\xVec) = w]\leq \delta_2 ~.$$

Thus, we have
\begin{align*}
    &\Pr_{\xVec \gets \xVecVar}[\xVec \in \BigOut(\delta_1) \mid \xVec \in \Low(\delta)] \\
    &= \sum_{w\in T} \Pr_{\xVec \gets \xVecVar}[\xVec \in \BigOut(\delta_1) \wedge t(\xVec) = w\mid \xVec \in \Low(\delta)]\\
    &= \sum_{w\in T} \Pr_{\xVec \gets \xVecVar}[\xVec \in \BigOut(\delta_1)\mid \xVec \in \Low(\delta)  \wedge t(\xVec) = w] \cdot \Pr_{\xVec \gets \xVecVar}[t(\xVec)=w \mid \xVec \in \Low(\delta)]\\
    &\leq \sum_{w\in T} \delta_2 \cdot \Pr_{\xVec \gets \xVecVar}[t(\xVec)=w \mid \xVec \in \Low(\delta)] \leq \delta_2 ~.
\end{align*}
\end{proof}

\begin{claim}\label{clm:ktam}
Let $x\in \High(\delta)$, then we have
$$\Pr[\zVecVar=\yVec]\geq (1-(1-\delta)^k)^n \cdot \frac{f(\yVec)}{\mu} \cdot \Pr[\xVecVar=\yVec] ~.$$
\end{claim}
\begin{proof}
Consider $\E_{k,\pfix{y}{i}}$ to be the event that $\RTamk$ outputs one of its first $k$ samples, when performed on $\pfix{y}{i}$. Then, it holds that
$$\Pr[\E_{k,\pfix{y}{i}}] = 1-(1-\fhat{y}{i})^k \geq 1 - (1-\delta)^k ~.$$
On the other hand, we know that $\Pr[\zVecVar_{i+1} = y_{i+1} \mid \pfix{y}{i} \wedge \E_{k,\pfix{y}{i}}]=\Pr[\yVecVar_{i+1} = y_{i+1} \mid \pfix{y}{i}]$. Thus, we have
\begin{align*}
    \Pr[\zVecVar_{i+1} = y_{i+1} \mid \pfix{y}{i}] &\geq \Pr[\zVecVar_{i+1} = y_{i+1} \mid \pfix{y}{i} \wedge \E_{k,\pfix{y}{i}}]\cdot \Pr[\E_{k,\pfix{y}{i}}]\\
    &=  \Pr[\yVecVar_{i+1} = y_{i+1} \mid \pfix{y}{i}]\cdot \Pr[\E_{k,\pfix{y}{i}}]\\
    &\geq \Pr[\yVecVar_{i+1} = y_{i+1} \mid \pfix{y}{i}]\cdot (1-(1-\delta)^k)^n.
\end{align*}
By multiplying these inequalities for $i\in [n]$ we get $\Pr[\zVecVar=\yVec]\geq (1-(1-\delta)^k)^n \cdot \Pr[\yVecVar=\xVec] ~.$
\end{proof}
\begin{claim}
For $\delta_1 \cdot \delta_2 = \delta$, it holds that
$$\mu_k  \geq \sum_{\yVec \in \Supp(\xVecVar)}\left(\frac{f(\yVec)}{\mu}\right)^p\cdot\Pr[\xVecVar=\yVec]\cdot f(\yVec)-\frac{\delta_1+\delta_2}{\mu} -n\cdot (1-\delta)^k ~.$$
\end{claim}

\begin{proof}
Let 
$$\mu' =  \sum_{\yVec \in \Low(\delta) \cap \SmallOut(\delta_1)} \left(\frac{f(\yVec)}{\mu}\right)^p\cdot\Pr[\xVecVar=\yVec]\cdot f(\yVec) ~,  $$
$$\text{and~~~} \mu'' =  \sum_{\yVec \in \Low(\delta) \cap \BigOut(\delta_1)} \left(\frac{f(\yVec)}{\mu}\right)^p\cdot\Pr[\xVecVar=\yVec]\cdot f(\yVec) ~.$$
By Claim, \ref{clm:main} we have
\begin{align*}
\Ex[f(\zVecVar)] &\geq \sum_{\yVec \in \Supp(\xVecVar)}\left(\frac{\Pr[\zVecVar^{[n]}=\yVec]}{\Pr[\xVecVar=\yVec]}\right)^p\cdot\Pr[\xVecVar=\yVec]\cdot f(\yVec)\\
            &\geq \sum_{\yVec \in \High(\delta)} \left(\frac{\Pr[\zVecVar^{[n]}=\yVec]}{\Pr[\xVecVar=\yVec]}\right)^p\cdot\Pr[\xVecVar=\yVec]\cdot f(\yVec)\\
            \text{ (by Claim \ref{clm:ktam}) }~ & \geq \sum_{\yVec \in \High(\delta)}  (1-(1-\delta)^k)^{n\cdot p}\cdot\left(\frac{f(\yVec)}{\mu}\right)^p\cdot\Pr[\xVecVar=\yVec]\cdot f(\yVec)\\
            &= (1-(1-\delta)^k)^{n\cdot p}\cdot\left(\sum_{\yVec \in \Supp(\xVecVar)}\left(\frac{f(\yVec)}{\mu}\right)^p\cdot\Pr[\xVecVar=\yVec]\cdot f(\yVec)-\mu' - \mu''\right).
\end{align*}
 We have $\mu' \leq {\delta_1^{1+p}}/{\mu^p}\leq {\delta_1}/ {\mu}$, because $f(\yVec)\leq \delta_1$ for all $\yVec \in \SmallOut(\delta_1)$. Also, by  Claim \ref{clm:LowBig}, we get

$$\mu''\leq \sum_{\yVec \in \Low(\delta) \cap \BigOut(\delta_1)} \left(\frac{1}{\mu}\right)^p\cdot\Pr[\xVecVar=\yVec] \leq \frac{\delta_2}{\mu^p}\leq \frac{\delta_2}{\mu} ~.$$
Therefore, we have 
\begin{align*}
\Ex[f(\zVecVar)] &\geq (1-(1-\delta)^k)^{n\cdot p}\cdot\left(\sum_{\yVec \in \Supp(\xVecVar)}\left(\frac{f(\yVec)}{\mu}\right)^p\cdot\Pr[\xVecVar=\yVec]\cdot f(\yVec)-\frac{\delta_1+\delta_2}{\mu}\right)\\
\text{(by Bernoulli inequality)~~}&\geq (1-{n\cdot}(1-\delta)^k)\cdot\left(\sum_{\yVec \in \Supp(\xVecVar)}\left(\frac{f(\yVec)}{\mu}\right)^p\cdot\Pr[\xVecVar=\yVec]\cdot f(\yVec)-\frac{\delta_1+\delta_2}{\mu}\right)\\
&\geq \sum_{\yVec \in \Supp(\xVecVar)}\left(\frac{f(\yVec)}{\mu}\right)^p\cdot\Pr[\xVecVar=\yVec]\cdot f(\yVec)-\frac{\delta_1+\delta_2}{\mu} -n\cdot (1-\delta)^k.
\end{align*}
\end{proof}

In order to conclude the proof of Claim~\ref{clm:effAttack}, we can set  $\delta_1 = \delta_2 = \sqrt{\delta}$ and let $\delta \leq (\eps\mu/4)^2$. Then, given that we have $k \geq \frac{16\ln (2n /\eps)}{\eps^2\mu^2}$, we  get
\begin{align*}
    \Ex[f(\zVecVar)] \geq \sum_{\yVec \in \Supp(\xVecVar)}\left(\frac{f(\yVec)}{\mu}\right)^p\cdot\Pr[\xVecVar=\yVec]\cdot f(\yVec)- \frac{\eps}{2} - \frac{\eps}{2} ~.
\end{align*}

\end{proof}

\subsection{Obtaining $(k,p)$-Poisoning: Proof of Theorem \ref{thm:mainMPP} using Theorem \ref{thm:main}} \label{sec:proofMPP}
In this section, we formally prove Theorem \ref{thm:mainMPP} using Theorems \ref{thm:main}. We first prove the first part of theorem about the boolean property.

\begin{proof}[Proof of Theorem \ref{thm:mainMPP}, Part 1]
For a subset $C\subseteq [m]$ let $P_C = \set{P_i; i\in C}$ and $R_C$ be the subset of rounds where one of the parties in $P_C$ sends an example. Also for a subset $S\subseteq[n]$, we define $\mathbf{Bion}(S,p)$ to be a distribution over all the subsets of $S$, where each subset $S'\subseteq S$ hast the probability $p^{|S'|}\cdot(1-p)^{|S|-|S'|}$. Now, consider the covering  $\SVar$ of the set $[n]$ which is distributed equivalent to the following process. First sample a uniform subset $C$ of $[m]$ of size $k$. Then sample and output a set $S$ sampled from $\mathbf{Bion}(R_C,p)$.
$\SVar$ is clearly a $(p\cdot\frac{k}{m})$-covering. We use this covering to prove the theorem. For $j\in [n]$ let $w(j)$ be the index of the provider at round $j$ and let $\dist_{w(j)}$ be the designated distribution of the $j$th round and let $\xVecVar = \dist_{w(1)}\times\dots\times\dist_{w(n)}$.

We define a function $f:\Supp(\xVecVar)\to \set{0,1}$, which is a Boolean function and is $1$ if the output of the protocol has the  property $B$, and otherwise it is $0$. Now we use Theorem \ref{thm:main}. We know that $\SVar$ is a $(p\cdot\frac{k}{m})$-covering for $[n]$. Therefore of Theorem  \ref{thm:main}, there exist an $\poly(m/\eps)$ time tampering algorithm $\Tam_\eps$ that changes $\xVecVar$ to $\yVecVar \equiv \twist{\xVecVar}{\Tamp_\eps^{f,\Samp}}{\SVar}$ where    $\Ex[f(\yVecVar)]
    \geq \Ex[f(\yVecVar)]^{1-pk/m} - \eps.$
    
By an averaging argument, we can conclude that there exist a set $C\in[m]$ of size $k$ for which the distribution $\mathbf{Bion}(R_C,p)$ produces average output at least $\Ex[f(\yVecVar)]^{1-pk/m} - \eps$. Note that the measure of empty set in $\mathbf{Bion}(R_C,p)$ is exactly equal to $1-p$ which means with probability $1-p$  the adversary will not tamper with any of the blocks, therefore, the statistical distance $|\xVecVar -  \twist{\xVecVar}{\Tamp_\eps^{f,\Samp}}{\mathbf{Bion}(R_C,p)}|$ is at most $p$. This concludes the proof.
\end{proof}

Now we prove the second part using Theorem \ref{thm:main} and Lemma \ref{lem:biasGap}.
\begin{proof} [Proof of Theorem \ref{thm:mainMPP} part 2]
Now we prove the second part. The second part is very similar to first part except that the function that we define here is a real valued function. Consider the function $f_2:\Supp(\xVecVar)\to [0,1]$
which is defined to be the risk of the output hypotheses. Now by Theorem \ref{thm:main} and Lemma \ref{lem:biasGap}, we know that there is tampering algorithm $\Tam_\eps$ that changes $\xVecVar$ to $\yVecVar \equiv \twist{\xVecVar}{\Tamp_\eps^{f_2,\Samp}}{\SVar}$ such that    
$$\Ex[f_2(\yVecVar)] 
\geq \mu_2 + \frac{p\cdot k}{2m}\cdot \nu -\eps.$$
By a similar averaging argument we can conclude the proof.
\end{proof}

\bibliographystyle{alpha}
\bibliography{Biblio/refs}
\appendix
\section{Some Useful Inequalities}

The following well-known variant of the inequality for the  arithmetic mean and the geometric mean could be derived from the Jensen's inequality.
\begin{lemma}[Weighted AM-GM inequality]
\label{lem:AM-GM}
For any $n \in \N$, let $z_1,...,z_n$ be a sequence of non-negative real numbers and let $w_1,...,w_n$ be such that $w_i \geq 0$ for every $i \in [n]$ and $\sum_{i=1}^n w_i = 1$. Then, it holds that
\begin{align*}
    \sum_{i=1}^n w_iz_i \geq \prod_{i=1}^n z_i^{w_i} ~.
\end{align*}
\end{lemma}

\remove{
\begin{lemma}[H\"{o}lder's inequality] \label{lem:Holder}
For any $n,m \in \N$, let $(z^j_i)_{i \in [n],j \in [m]}$ be a sequence of non-negative real numbers and let $w_1,...,w_m$ be such that $1/w_j \geq 1$ for every $j \in [m]$ and $\sum_{j=1}^m w_j = 1$. Then we have that:
\begin{align*}
    \prod_{j=1}^m \left(\sum_{i=1}^n z^j_i\right)^{w_j} \geq \sum_{i=1}^n \prod_{j=1}^m \left(z^j_i\right)^{w_j}
\end{align*}
\end{lemma}
}

The following lemma provides a tool for lower bounding the gap between the two sides of Jensen's inequality, also known as the Jensen gap.

\begin{lemma}[Lower bound for Jensen gap \cite{JensenSharp}] \label{lem:JensenGap}
Let $\alphaVar$ be a real-valued random variable, $\Supp(\alphaVar) \se [0,1]$, and $\Ex[\alphaVar]=\mu$. Let $\vphi(\cdot)$ be  twice differentiable on $[0,1]$, and let $h_b(a) = \frac{\vphi(a) - \vphi(b)}{(a-b)^2} - \frac{\vphi'(a)}{a-b}$. Then,
$$  \Ex[\vphi(\alphaVar)] - \vphi(\mu) \geq \Var[\alphaVar]\cdot \inf_{a \in [0,1]} \set{h_\mu(a)} ~.$$
\end{lemma}

\subsection{Relating the Bias to the Variance}\label{appendix:variance}
\remove{
{\bf  Special case of Boolean functions.} First we prove Lemma \ref{lem:biasGap} for the Boolean case (i.e., Inequality \ref{eq:GapBoolean}). In particular, we will use Inequalities \ref{in:1} and \ref{in:2} of Lemma \ref{lem:ineqMu} as follows.
\begin{align*}  
\mu'-\mu = \mu^{-p} \cdot \Exp [ f(\xVecVar)^{1+p} ] - \mu &= \mu^{1-p}-\mu \\
\text{ (by Inequality \ref{in:2}) } &\geq \mu \cdot \ln(1/\mu) \cdot p = \frac{\ln(1/\mu)}{1-\mu} \cdot p \cdot \nu \\
\text{ (by Inequality \ref{in:1}) } &\geq p \cdot \nu ~.
\end{align*}

The above simple argument does not work for general real-valued functions. However, the following argument still shows that there is a gap between $\mu'$ and $\mu$ based on the variance of  $f(\xVecVar)$ and $p$.

{\bf  The case of real-valued functions.} 
}
We first prove a lemma that shows the connection of bias to variance. 
Then, using this lemma we immediately get a $\Omega(p \cdot \Var[f(\xVecVar)])$ lower bounds for the bias achieved by  the attacker of Theorem \ref{thm:main} for the general case of real-valued functions and arbitrary $p$-covering set distribution $\SVar$. 

\begin{lemma} \label{lem:biasGap}
Let $\alphaVar$ be any real-valued random variable over $\Supp(\alphaVar) \se [0,1]$, and $p \in [0,1]$. Let $\mu=\Ex[\alphaVar]$ be the expected value of $\alphaVar$, $\nu = \Var[\alphaVar]$ be the variance of $\alphaVar$. Then, it holds that
\begin{equation*} \label{eq:GapReal}
\mu^{-p}\cdot \Ex[\alphaVar^{1+p}]-\mu \geq \frac{p\cdot (p+1)}{2\cdot \mu^p} \cdot \nu \geq  \frac{p}{2} \cdot \nu ~.
\end{equation*}
\end{lemma}
\begin{proof}
We use Lemma \ref{lem:JensenGap} by letting $\vphi(x) = x^{1+p}$. Thus, we have to minimize the following function on $x \in [0,1]$,
$$g_\mu(x) = \left(x^{1+p} - \mu^{1+p} - (1+p) \cdot \mu^p \cdot (x-\mu) \right) / (x-\mu)^2 ~.$$
We now prove that the minimum happens on $x=1$. Note that the function $g_\mu(x)$ is continues on $[0,\mu)$ and $(\mu,0]$ and the limit exists at $x=\mu$ and is equal to $1/2\cdot p\cdot (1 + p)\cdot \mu^{-1 + p}$. Therefore if we show that $g'_\mu$ is negative for $x\in[0,\mu)\cup(\mu,1]$ it implies that, $\forall x\in[0,1] g(x) \geq g(1)$. We have
\begin{align*}
g'_\mu(x)&=\frac{(p - 1)\cdot \mu^{p + 1} - (p + 1)\cdot x \cdot \mu^p + (p + 1)\cdot \mu\cdot x^p - (p - 1)\cdot x^{p + 1}}{(\mu - x)^3}\\
(\text{using } c=x/\mu)~~ &= \mu^{p-2}\cdot\frac{(p - 1) - (p + 1)\cdot c + (p + 1)\cdot c^p - (p - 1)\cdot c^{p + 1}}{(1 - c)^3} ~.
\end{align*}
We prove that the numerator $q(c)=(p - 1) - (p + 1)\cdot c + (p + 1)\cdot c^p - (p - 1)\cdot c^{p + 1}$ is positive for $c>1$ and negative for $0<c<1$. For $c>0$, we have
\begin{align*}
q'(c) &= -(1+p) + (p+1)\cdot p \cdot c^{p-1} + (1-p)\cdot (p+1) \cdot c^p\\
      &= (1+p)\cdot(p \cdot c^{p-1} + (1-p)\cdot c^p -1)\\
(\text{by AM-GM inequality of Lemma \ref{lem:AM-GM}}) & \geq (1+p)\cdot(c^{p\cdot(p-1)}\cdot c^{(1-p)\cdot p} -1)\\
&= 0 ~.
\end{align*}
Therefore, $q$ is increasing for $c>0$ which implies $\forall c\in [0,1], q(c)< q(1) = 0$ and $\forall c>1, q(c)>q(1) = 0$. We have $\forall x\in[0,\mu)\cup(\mu,1], g'(x) \leq 0$.
Therefore we have
\begin{align}\label{ineq:01}
    \forall x\in[0,1], g_u(x) \geq g_u(1) ~.
\end{align}
Now we prove that $g_\mu(1) \geq \frac{p(1+p)}{2}$. Consider the following function,
$$w(\mu) = g_\mu(1) = \left(1 - \mu^{1+p} - (1+p) \cdot \mu^p \cdot (1-\mu) \right) / (1-\mu)^2~.$$
We will show that $q$ is a decreasing function for $\mu\in [0,1]$. We have
$$w'(\mu) = \frac{p\cdot(1-\mu^2)\cdot \mu^{p-1} + p^2 (1-\mu)^2\cdot \mu^{p-1} + 2 \cdot(\mu^p-1))}{(-1 + \mu)^3} ~.$$
We will show that the numerator $s(\mu) = p\cdot(1-\mu^2)\cdot \mu^{p-1} + p^2 (1-\mu)^2\cdot \mu^{p-1} + 2 \cdot(\mu^p-1)$ is negative for $\mu\in [0,1] ~.$ We have $s'(\mu)=p (p^2-1)\cdot(1-\mu)^2\cdot\mu^{p-2}$ which is negative for $\mu\in[0,1]$. This  implies that $\forall\mu\in [0,1], s(\mu) \geq s(1) = 0$. Therefore, $w$ is a decreasing function, and we obtain 
\begin{align}\label{ineq:02}\forall \mu \in [0,1], g_\mu(1)= w(\mu) \geq \lim_{u\to 1} w(u) =\frac{p(1+p)}{2} ~.\end{align}
 Now, we conclude that
\begin{align*}
        \mu^{-p} \cdot \Exp [ \alphaVar^{1+p} ] - \mu &= \mu^{-p}\left(\Exp [ \alphaVar^{1+p} ] - \mu^{1+p}\right)\\
        \text{(by Lemma \ref{lem:JensenGap})}~~~&\geq  \mu^{-p}\left(\inf_{x\in[0,1]}\set{g_\mu(x)}\cdot \nu\right) \\
        \text{(by Inequality \ref{ineq:01})}~~&\geq  \mu^{-p}\cdot g_\mu(1)\cdot \nu \\
        \text{(by Inequality \ref{ineq:02})}~~&\geq \frac{p\cdot (1+p)}{2\cdot \mu^p} \cdot \nu ~.
\end{align*}
\end{proof}
Now, using Lemma \ref{lem:biasGap} and \ref{thm:main}, we immidiately get Corollary \ref{cor:biasToVar}  

\end{document}